\documentclass[conference]{IEEEtran}

\usepackage{amssymb}
\usepackage{amsmath}
\usepackage{amsthm}
\usepackage{amsbsy}
\usepackage{pgfgantt}
\usepackage{graphicx}
\graphicspath{ {figures/} }
\usepackage{subfigure}
\usepackage{wrapfig}
\usepackage{booktabs} 
\usepackage{subfiles}
\usepackage{latexsym} 
\usepackage[colorlinks,citecolor=red]{hyperref}
\usepackage{array}
\usepackage{xcolor}
\usepackage{tikz}
\usetikzlibrary{shapes}
\usetikzlibrary{calc}

\usepackage{times}

\usepackage[numbers]{natbib}
\usepackage{multicol}

\usepackage{algorithm}
\usepackage[noend]{algpseudocode}
\makeatletter
\def\BState{\State\hskip-\ALG@thistlm}
\makeatother

\makeatletter
\let\OldStatex\Statex
\renewcommand{\Statex}[1][3]{%
  \setlength\@tempdima{\algorithmicindent}%
  \OldStatex\hskip\dimexpr#1\@tempdima\relax}
\makeatother

\newtheorem{theorem}{Theorem}

\newtheorem{definition}{Definition}

\newcommand{\state}{x}
\newcommand{\action}{u}
\newcommand{\horizon}{T}
\newcommand{\strategy}{U}

\begin{document}

\title{Potential iLQR: A Potential-Minimizing Controller for Planning Multi-Agent Interactive Trajectories}



\author{\authorblockN{Talha Kavuncu}
\authorblockA{Aerospace Engineering\\
UIUC\\
kavuncu2@illinois.edu}
\and
\authorblockN{Ayberk Yaraneri}
\authorblockA{Aerospace Engineering\\
UIUC\\
ayberky2@illinois.edu}
\and
\authorblockN{Negar Mehr}
\authorblockA{Aerospace Engineering\\
UIUC\\
negar@illinois.edu}}

\maketitle

\begin{abstract}
Many robotic applications involve interactions between multiple agents where an agent's decisions affect the behavior of other agents. Such behaviors can be captured by the equilibria of differential games which provide an expressive framework for modeling the agents' mutual influence. However, finding the equilibria of differential games is in general challenging as it involves solving a set of coupled optimal control problems. In this work, we propose to leverage the special structure of multi-agent interactions to generate interactive trajectories by simply solving a single optimal control problem, namely, the optimal control problem associated with minimizing the potential function of the differential game. Our key insight is that for a certain class of multi-agent interactions, the underlying differential game is indeed a potential differential game for which equilibria can be found by solving a single optimal control problem. We introduce such an optimal control problem and build on single-agent trajectory optimization methods to develop a computationally tractable and scalable algorithm for planning multi-agent interactive trajectories. We will demonstrate the performance of our algorithm in simulation and show that our algorithm outperforms the state-of-the-art game solvers. To further show the real-time capabilities of our algorithm, we will demonstrate the application of our proposed algorithm in a set of experiments involving interactive trajectories for two quadcopters.

\end{abstract}

\IEEEpeerreviewmaketitle


\section{Introduction}\label{sec:intro}
Many robotic applications involve interactions between multiple agents. For instance, two quadcopters may need to interact and implicitly coordinate to successfully navigate in a shared three-dimensional space (Fig.~\ref{fig:vehicles_front}). An autonomous car may need to interact with human-driven cars or pedestrians to navigate an intersection. Planning in such interactive settings is in general challenging due to the feedback interactions between the agents. An agent's state and action will affect the state and actions of the other agents, i.e., agents are coupled by their intentions and actions. In this work, we demonstrate that by leveraging the structure that is inherent in such interactive settings, we can resolve these couplings, and agents can plan interactive trajectories by solving an optimal control problem. 


\begin{figure}
\includegraphics[width=1\linewidth]{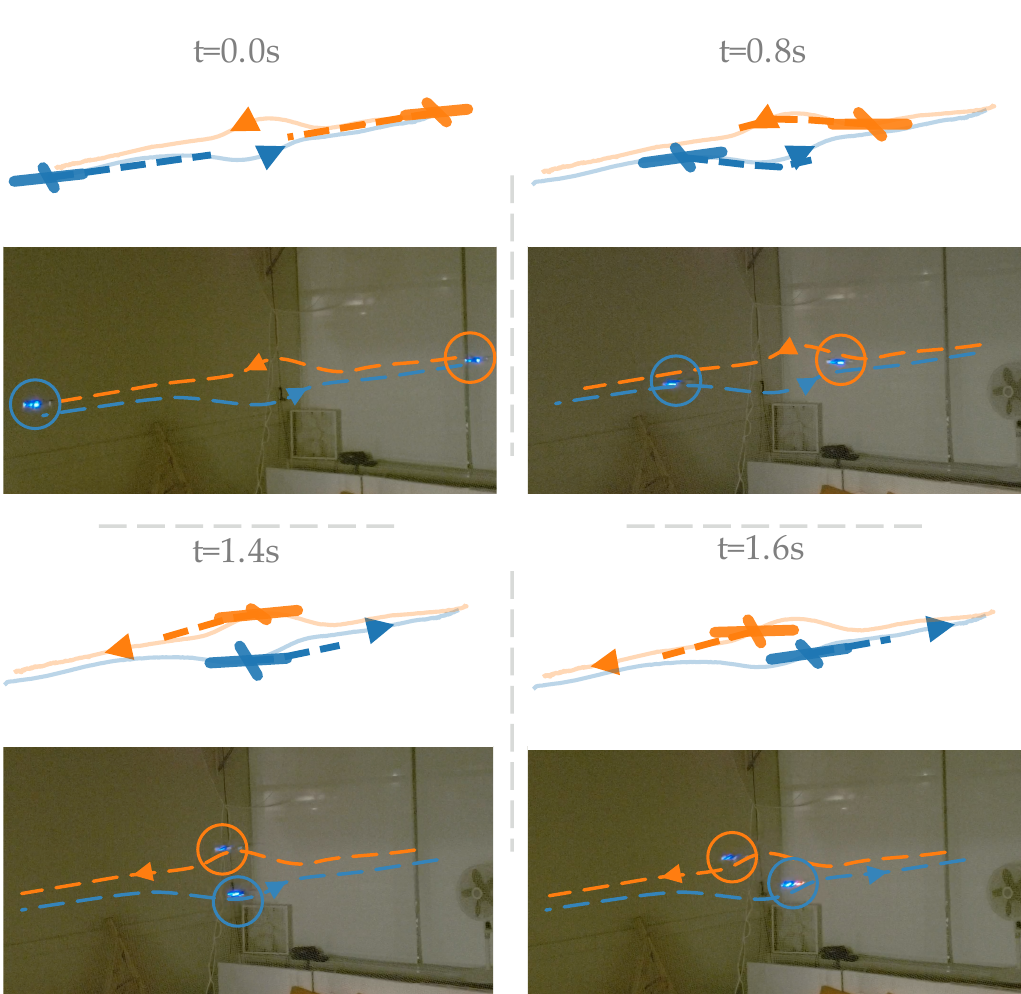}
\caption{Demonstration of our interactive trajectory planning algorithm in an experiment involving two quadcopters. The two quadcopters start at roughly the same altitude and switch their positions such that the starting position of one becomes the goal position of the other. The quadcopters exhibit intuitive interactive trajectories and change altitudes for avoiding collisions with each other.}
\label{fig:vehicles_front}
\end{figure}



It has been shown that feedback interactions in interactive settings can be captured by differential games~\cite{sadigh2016planning,sadigh2016information}. Every agent is a utility maximizer seeking to maximize their own utility over a horizon of time while an agent's utility can depend on the state and actions of all the agents. In such settings, the mutual influence of the agents as well as the outcome of the interaction is best represented by an equilibrium of the underlying differential game. However, 
finding the equilibrium strategies of such games is in general challenging as it involves solving a set of coupled optimal control problems. Consequently, most differential games do not have analytical solutions, and even their numerical solutions are not scalable~\cite{fabrikant2004complexity,bacsar1998dynamic}.



In this work, we propose to leverage the special structure of multi-agent interactions to generate interactive trajectories by simply solving a \emph{single} optimal control problem. 
Our key insight is that for a certain class of multi-agent interactions, the underlying differential game is indeed a potential differential game. Potential games are a class of games for which a Nash equilibrium always exists, and the Nash equilibrium can be found by solving a single optimization problem~\cite{monderer1996potential}. Thus, we can employ a standard single-agent trajectory optimization method such as iLQR~\cite{li2004iterative} for planning multi-agent interactive trajectories. 

We will prove that a class of multi-agent interactions, namely, interactions where the mutual couplings between the agents are symmetric, are indeed potential differential games. For such games, we introduce the optimal control problem associated with minimizing the potential of the game whose solution is the equilibrium trajectory of the original interaction game. Using this result, we develop a computationally tractable algorithm for interactive trajectory planning. 
We will compare the performance of our algorithm with the state-of-the-art in interactive trajectory planning both in terms of the quality of the trajectories as well as the computational tractability. To further show the real-time capabilities of our algorithm, we will demonstrate the application of our framework in a set of experiments involving interactive trajectories for two quadcopters.

\section{Related Work}\label{sec:related}
\subsection{Interactive Trajectory Planning}

The common approach to interactive trajectory planning is for a robot to make predictions of the future trajectories of other agents and plan reactively~\cite{ziebart2009planning,schmerling2018multimodal,nishimura2020risk,vitus2013probabilistic,bai2015intention}. Planning reactively will make the agents decoupled and simplify the control problem. Nevertheless, throughout this decoupling, agents lose the capability to affect each other. To capture this, it has been shown that interactive settings can be modeled by equilibria of differential games~\cite{sadigh2016planning,sadigh2016information,turnwald2016understanding}. Several methods have been proposed for finding equilibria in interaction games. In~\cite{sadigh2016planning}, a Stackelberg equilibrium was considered where one agent is the leader, and the other agent is the follower. To find Nash equilibria of the interaction game, in~\cite{fisac2019hierarchical}, a hierarchical decomposition of the underlying game into strategic and tactical games was proposed. Iterative best response algorithms were developed for capturing interactions in racing problems and autonomous driving settings~\cite{schwarting2019social,wang2019game,spica2020real}. In~\cite{schwarting2021stochastic}, iterative dynamic programming in Gaussian belief space was used to solve for equilibria of a game-theoretic continuous POMDP.

\subsection{Approximate Solutions to Differential Games }

To find equilibria of general differential games, sequential linear-quadratic methods were proposed for two-player zero-sum differential games~\cite{mukai2000sequential,tanikawa2012local}.
To enable scalable interactive trajectory planning for a broad class of differential games, recently, a local iterative algorithm was proposed in~\cite{fridovich2020efficient} where the analytic solution to the Linear Quadratic games~\cite{bacsar1998dynamic} was exploited for approximating the equilibria of general-sum differential games. This algorithm builds upon the iterative
linear-quadratic regulator (iLQR)~\cite{li2004iterative}, and at every iteration, solves for the LQ game that results from linearizing the system dynamics and finding a quadratic approximation of the agents' cost functions. A similar iterative method was proposed in~\cite{wang2020game} for planning interactive trajectories in the presence of uncertainties where equilibria of risk-sensitive dynamic games were sought. In~\cite{cleac2019algames}, a solver was developed for interactive trajectory planning in the presence of general nonlinear state and input constraints.

\subsection{Potential Games}

The literature on potential games is mostly focused on static games. A class of static potential games with pure Nash equilibria was identified in~\cite{rosenthal1973class}.
Later, potential games were also introduced in~\cite{monderer1996potential}. Because of the appealing properties of potential games, potential games have had applications in various control and resource allocation problems~\cite{zazo2014control, zazo2015new, arslan2007autonomous,marden2009cooperative,candogan2010near,fonseca2018potential}. We argue that potential games, in the form of potential differential games, can be further utilized for trajectory planning in multi-agent settings.


\section{Problem Formulation}\label{sec:problem}
We assume that we have $N$ agents. For each agent $i$, $ 1 \leq i\leq N$, the vector $u_i(t) \in \mathbb{R}^{m_i}$ represents the control input of agent $i$ at time $t$, where $m_i$ is the dimension of the control input of agent $i$. Similarly, we let $x_i(t)\in \mathbb{R}^{n_i}$ denote the state of agent $i$ at time $t$, where $n_i$ is the dimension of the state space of agent $i$. We let $x(t)=(x_1(t),\cdots, x_N(t))$ denote the concatenated vector of all agents' states at time $t$, and $n$ be the dimension of the state vector $x(t)$. We further use $u(t)=(u_1(t),\cdots,u_N(t))$  to denote the vector of all agents' control inputs at time $t$. The overall system dynamics are
\begin{equation}\label{eq:dynamics}
    \Dot{\state}(t)=f(\state(t),\action(t),t).
\end{equation}
We consider open-loop control inputs, i.e.,\ control inputs that are only a function of the system's initial state $\state_0$ and time $t$:
\begin{equation}\label{eq:open-loop-control}
    \action_i (t)=\action_i(\state_0,t).
\end{equation}
Although the open-loop assumption may seem restrictive, in many practical applications, open-loop trajectory planning algorithms are applied in a receding-horizon fashion to approximate closed-loop feedback policies\footnote{We acknowledge that in general, receding horizon application of open-loop equilibrium strategies may not be enough for finding close-loop equilibrium policies due to the difference between the information structure of open-loop and closed-loop policies. However, for many interactive trajectory planning settings, this is a valid approximation.}. For each agent $i$, we let the set of Borel measurable functions
$\strategy_i:=\{\action_i|\;\action_i: T \rightarrow \mathbb{R}^{m_i}\}$ represent the open-loop strategy space of player $i$ that maps time to the player $i$'s control input.  Moreover, we let $\strategy_{-i}:=  \strategy_1 \times \cdots \times \strategy_{i-1} \times \strategy_{i+1} \times \cdots \times \strategy_N$ represent the open-loop strategy space of all agents except agent $i$.
We write $(u_i,u_{-i}^*)$ to denote the vector 
\begin{align}
    (u_1^*, \cdots, u_{i-1}^*,u_i,u_{i+1}^*, \cdots, u_N^*) \in U,
\end{align}
where $U$ is the strategy space of all agents.

We assume that each agent $i$ minimizes a cost function $J_i(\cdot)$, where $J_i$ depends on the initial state $x_0$ and the agents' control signals $u_1,\cdots, u_N$ through the following
\begin{equation}\label{eq:cost}
    J_i(x_0,u_1,\cdots,u_N)=\int_{0}^{\horizon}L_i(\state(t),\action(t),t) dt
    +S_i(\state(\horizon)),
\end{equation}
where $T$ is a finite time horizon, and $L_i$ and $S_i$ are the running and terminal costs of agent $i$ respectively. 



In a compact form, we describe our differential game by
\begin{align}
   \Gamma_{\state_0}^\horizon := (N,\{\strategy_i\}_{i = 1}^N,\{J_i\}_{i=1}^N,f),
\end{align}
where $x_0$ is the initial state of the system.
In a multi-agent setting, since each agent $i$ seeks to optimize their own cost $J_i(.)$, the outcome of interaction is best represented via a notion of equilibrium. Among the various notions of equilibria, we look for Nash equilibria which characterize the interaction outcome in non-cooperative multi-agent settings. 


\begin{definition}\label{def:nash} Given a differential game $\Gamma_{\state_0}^\horizon$, a control signal $u^* = (u_1^*, \cdots, u_N^*)$ is an open-loop Nash equilibrium if for every agent $i$, we have
\begin{equation}\label{eq:nash-def}
    J_i(x_0, \action^*(\cdot))\leq J_i(x_0, \action_i(\cdot),\action_{-i}^*(\cdot)).
\end{equation}
\end{definition}

Intuitively, at Nash equilibrium, no agent has any incentive for unilateral deviation from $\action_i^*(\cdot)$, i.e., when the control inputs of all the other agents $\action_{-i}^*(\cdot)$ are fixed, as~\eqref{eq:nash-def} suggests, agent $i$ will not benefit from changing its equilibrium control signal. In general, finding Nash strategies $\action^*$ satisfying~\eqref{eq:nash-def} is challenging since it involves solving $N$ coupled optimal control problems.


\section {Potential Differential Games}\label{sec:potential}
In this section, we introduce potential differential games, and then in the next section, we discuss how we can leverage potential differential games for tractable interactive trajectory planning in multi-agent settings. While finding Nash equilibria is in general challenging, there exists a class of differential games called potential differential games to which we can associate an optimal control problem (OCP) whose solutions are open-loop Nash equilibria for the original game $\Gamma_{\state_0}^\horizon$~\cite{fonseca2018potential}. 

\begin{definition}\label{def:potential_diff_game} (cf.~\cite{fonseca2018potential})
A differential game $\Gamma_{\state_0}^\horizon$ is a potential differential game if there exists an optimal control problem whose solutions are Nash equilibria of the game $\Gamma_{x_0}^T$.
\end{definition}

 This problem reduction allows us to benefit from the existing planning methods for solving single-player optimal control problems to calculate Nash equilibria in interactive settings. 
In interactive trajectory planning, agents' dynamics are normally decoupled, and each agent's state update is governed by its own control inputs and its own dynamics. The coupling between the agents occurs due to the coupling between the agents' cost functions. For example, in navigation problems, the coupling between the agents arises from the inter-agent collision avoidance costs. We leverage this property, and in the rest of this paper, assume that for each agent $i$, we have
\begin{equation}\label{eq:dyn-decoupled}
    \dot{\state}_i(t)=f_i(\state_i(t),\action_i(t),t),
\end{equation}
where $f_i$ is the dynamics of the $i_{\text{th}}$ agent.

Under decoupled dynamics~\eqref{eq:dyn-decoupled}, it has been shown in~\cite{fonseca2018potential} that a differential game is a potential differential game if the following holds.

\begin{theorem}\label{theorem:potential-definition} 
For a differential game $ \Gamma_{\state_0}^T = (N,\{\strategy_i\}_{i = 1}^N,\{J_i\}_{i=1}^N,\{f_i\}_{i=1}^N)$, if for each agent $i$, the running and terminal costs have the following structure 
\begin{equation}\label{eq:running-cost-structure}
    L_i(\state(t),\action(t),t)=
    p(\state(t),\action(t),t)+
    c_i(\state_{-i}(t),\action_{-i}(t),t),
\end{equation}
and 
\begin{equation}\label{eq:terminal-cost-strucutre}
    S_i(\state(\horizon))=\bar{s}(\state(\horizon))+
    s_i(\state_{-i}(\horizon)),
\end{equation}
then, the open-loop control input $\action^*=(\action^*_1,\cdots, \action^*_N)$ that minimizes the following optimal control problem
\begin{equation}\label{eq:OCP}
\begin{aligned}
    \min_{\action(\cdot)} \quad & \int_{0}^{\horizon}p(\state(t),\action(t),t)dt+
    \bar{s}(\state(\horizon))\\
    \textrm{s.t.} \quad & \dot{\state}_i(t)=f_i(\state_i(t),\action_i(t),t),
\end{aligned}
\end{equation}
is an open-loop Nash equilibrium of the differential game $\Gamma_{x_0}^T$, i.e.,\
$\Gamma_{\state_0}^T$ is a potential differential game.
\end{theorem}

\begin{proof} See Appendix~\ref{sec:appendix}. \end{proof}
Note that Theorem~\ref{theorem:potential-definition} requires the running cost of every agent $i$ to be composed of a potential function $p$ and a term $c_i$ which has no dependence on the state and action of agent $i$. The potential function $p$ may depend on the entire state and input vector $\state$ and $\action$, but it is not agent-specific as it does not have any dependence on the agent's index $i$. On the other hand, the agent-specific term $c_i$ must depend only on the states and actions of all agents except agent $i$. In the next section, we will show how this decomposition can be achieved when the coupling between agents occurs due to collision avoidance cost terms. It is important to mention that in general, there are less restrictive conditions under which a differential game is a potential game, but we have only included the conditions that are relevant to interactive trajectory planning. Interested reader is referred to~\cite{fonseca2018potential} for further details.

Note that while a solution to~\eqref{eq:OCP} is always a Nash equilibrium of the original game $\Gamma_{x_0}^T$, there may exist other equilibria for the game $\Gamma_{x_0}^T$ which do not necessarily optimize~\eqref{eq:OCP}. In other words, solving optimal control~\eqref{eq:OCP} always provides a set of equilibria which is a subset of all equilibria of the game. Nevertheless, if a game is potential game, we are guaranteed that an equilibrium exists.




\section{Interactive Trajectory Planning}\label{sec:traj-planning}
 In this section, we discuss how Theorem~\ref{theorem:potential-definition} and the special structure of agents' cost functions can be leveraged for interactive trajectory planning. In multi-agent settings such as navigation, the cost function of each agent $i$ is typically composed of two types of cost terms: (i) Cost terms which are only dependent on the state and action of agent $i$ itself such as input and state tracking costs, and (ii) cost terms that capture the mutual couplings between the pairs of agents, such as collision avoidance costs, which are dependent on the state of the agent $i$ as well as other agents. For examples of these cost structures, see~\cite{sadigh2016planning,fridovich2020efficient,wang2020game,fisac2019hierarchical}. For instance, the agent-specific tracking cost can be composed of a running cost $C^{tr}_i$ and a terminal cost $C_{i,\horizon}^{tr}$ in the following form
 \begin{equation}\label{eq:tr_cost}
 \begin{aligned}
    C_{i}^{tr}(x_i,u_i) =&(\state_i-\state_{i}^\text{ref})^\intercal{Q}_i(\state_i-\state_{i}^\text{ref})+\\&
    (\action_i-\action_{i}^\text{ref})^\intercal{}{R}_i(\action_i-\action_{i}^\text{ref}),\\
    C_{i,\horizon}^{tr}(x_i,u_i)=&
    (\state_{i}(\horizon)-\state_{i}^\text{ref}(\horizon))^\intercal {Q}_i(\state_{i}(\horizon)-\state_{i}^\text{ref}(\horizon)),
 \end{aligned}
\end{equation}
where ${Q}_i$ and ${R}_i$ are weight matrices for penalizing state and control deviations from a reference trajectory $(\state_{i}^\text{ref}, \action_{i}^\text{ref})$. In a navigation setting, $\state_{i}^\text{ref}$ represents the goal state of agent $i$, and $\action_{i}^\text{ref}$ is the zero input signal for agent $i$ to minimize its control effort. Note that in general, the matrices $Q_i$ and $R_i$ can be time-variant. 

In addition to tracking cots, each agent $i$ has coupling cost terms too that create mutual impact between the agents such as avoiding collisions with other agents. For each agent $i$, we let the collision avoidance cost term $C_i^a$ be composed of pairwise collision avoidance costs $C^{a}_{ij}$ for all $j\neq i$. For each $j\neq i$, $C^{a}_{ij}$ penalizes agent $i$ for colliding with or getting close to agent $j$. We assume that pairwise collision avoidance terms $C^{a}_{ij}$ have the following structure

\begin{equation}\label{eq:ca_cost}
    C_{ij}^{a}(x_i,x_j)=\alpha_{ij}\left(d(\state_i,\state_j)\right)
\end{equation}
where $\alpha_{ij}$ is a function of the distance $d$ between the agents. Note that $d$ depends only on the states of agent $i$ and $j$. 
Hence, for agent $i$, the running and terminal costs become 
\begin{equation}\label{eq:running-cost-revised}
    L_i (\state, \action_i)
    =C_{i}^{tr}(\state_i,\action_i)+\sum_{j\neq i}^{N}C_{ij}^a(\state_i,\state_j)
\end{equation}
and
\begin{equation}\label{eq:terminal-cost-revised}
        S_i(\state_i(T))=C_{i,\horizon}^{tr}(x_{i}(T)).
\end{equation}
The above cost structures can be more general. For example, $C_i^\text{tr}$ can encode any other objective that agent $i$ cares about such as minimum time to reach, minimum fuel, etc. Similarly, the inter-agent coupling terms can be more complicated than collision avoidance costs. But for simplicity, in the presentation of the paper, we specifically consider tracking and collision avoidance costs.
Our key insight is that if the inter-agent collision avoidance costs~\eqref{eq:ca_cost} are symmetric for any two agents $i$ and $j$, i.e., $C_{ij}^{a}(x_i,x_j)=C_{ji}^{a}(x_j,x_i)$ for all pair of agents; then, the game $\Gamma_{\state_0}^T$ is a potential differential game. In other words, if any two agents penalize for collisions with each other similarly, the game is a potential game. Note that this does not imply that all agents penalize collisions similarly. We only need any two agents to penalize for getting close to each other similarly. In other words, the agents' sensitivity to collisions with each other should be symmetric. 


\begin{figure}
\centering
\includegraphics[width=0.3\textwidth]{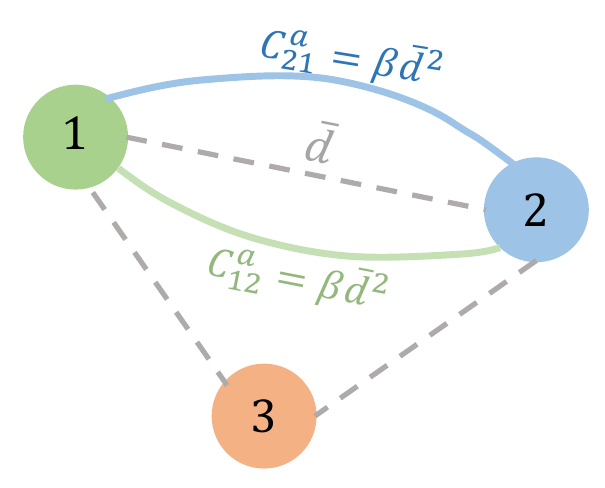}
\caption{Agent 1 penalizes for collision with agent 2 the same way that agent 2 penalizes for collision with agent 1. For example, if collision avoidance costs are in the form of the quadratic of the distance between the two agents, we must have $C_{12}^a=C_{21}^a=\beta \bar{d}^2$ where $\bar{d}$ is the distance between agents 1 and 2 and $\beta$ is the weight for avoiding collisions. Agents 1 and 2 are symmetric in how they penalize for collisions. If collision avoidance costs are symmetric for all pair of agents, then the underlying game is a potential differential game.}
\label{fig:three-agent-example}
\end{figure}

We first show this through an example. Consider a differential game with three agents (see Fig.~\ref{fig:three-agent-example}), and cost structure~\eqref{eq:running-cost-revised} and~\eqref{eq:terminal-cost-revised}. We show that the game is potential if inter-agent cost terms are symmetric. We let $p$ and $\bar{s}$ in Theorem~\ref{theorem:potential-definition} be
\begin{equation}\label{eq:3-agent-p}
\begin{aligned}
        p(\state,\action) &=C_1^{tr}(\state_1,\action_1)+C_2^{tr}(\state_2,\action_2)+C_3^{tr}(\state_3,\action_3)\\
        &+C^a_{12}(x_1,\state_2)+C^a_{13}(\state_1,\state_3)+C^a_{23}(\state_2,\state_3),
\end{aligned}
\end{equation}
\begin{equation}\label{3-agent-s-bar}
    \begin{aligned}
     \bar{s}(\state,T)&=C_{1,\horizon}^{tr} (\state_1(T))+C_{2,\horizon}^{tr}(\state_2(T))+C_{3,\horizon}^{tr}(\state_3(T)).
    \end{aligned}
\end{equation}
For each agent $i$, we define the term $c_i$ in~\eqref{eq:running-cost-structure} to be
\begin{equation}
\begin{aligned}
        &c_1(\state_2,\state_3,\action_2,\action_3)=\\
        &-C^{tr}_2(\state_2,\action_2)-C^{tr}_3(\state_3,\action_3)-C^a_{23}(\state_2,\state_3),\\ 
        &c_2(\state_1,\state_3,\action_1,\action_3)=\\
        &-C^{tr}_1(\state_1,\action_1)-C^{tr}_3(\state_3,\action_3)-C^a_{13}(\state_1,\state_3),\\
        &c_3(\state_1,\state_2,\action_1,\action_2)=\\
        &-C_1^{tr}(\state_1,\action_1)-C^{tr}_2(\state_2,\action_2)-C^a_{12}(\state_1,\state_2).
\end{aligned}
\end{equation}
Likewise, for each agent $i$, we define the $s_i$ term in~\eqref{eq:terminal-cost-strucutre} to be:
\begin{equation}
    \begin{aligned}
            s_1(x_2(T),x_3(T))&=-C_{2,\horizon}^{tr}(x_2(T))-C_{3,\horizon}^{tr}(x_3(T)),\\
            s_2(x_1(T),x_3(T))&=-C_{1,\horizon}^{tr}(x_1(T))-C_{3,\horizon}^{tr}(x_3(T)),\\
            s_3(x_1(T),x_2(T))&=-C_{1,\horizon}^{tr}(x_1(T))-C_{2,\horizon}^{tr}(x_2(T)).
    \end{aligned}
\end{equation}
The potential function~\eqref{eq:3-agent-p} consists of the sum of the running costs of all agents and the sum of all pair-wise collision avoidance costs of agents with unordered pairs of distinct indices $\{i,j\}$. Similarly, $\bar{s}$ is the sum of the terminal costs of all agents. It is easy to show that if $C^a_{ij}(x_i,x_j)=C^a_{ji}(x_j,x_i)$ for any two agents $i$ and $j$, we have:

\begin{align}
     L_i(\state(t),\action(t),t) &=
    p(\state,\action,t)+
    c_i(\state_{-i},\action_{-i},t), \quad 1 \leq i \leq 3, \\
    S_i(\state(\horizon))&=\bar{s}(\state(\horizon))+
    s_i(\state_{-i}(\horizon)),  \quad 1 \leq i \leq 3.
\end{align}
Thus, using Theorem~\ref{theorem:potential-definition}, our game is a potential game. We can generalize this intuition through the following theorem:




\begin{figure*}[h!]
\includegraphics[width=\textwidth]{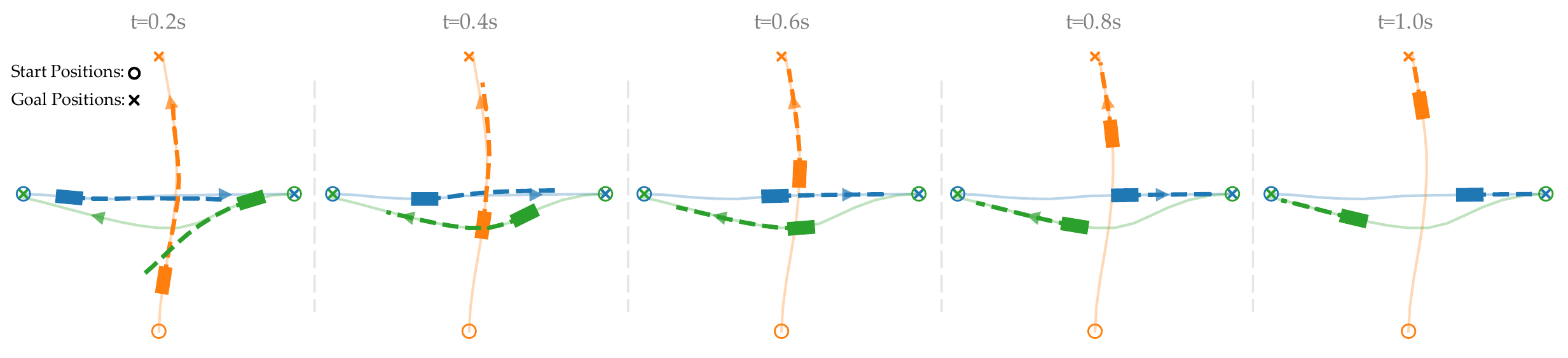}
\caption{The snapshots of the trajectories found by our algorithm for an intersection scenario. Agents move from their start positions (denoted by circles) and move towards their goal positions (denoted by cross signs). The agents manage to successfully avoid collisions and resolve conflicts when they get close to each other. Dashed lines represent the planned trajectories over the receding horizon while the solid lines represent the final resulting trajectories}
\label{fig:3_agents_mpc}
\end{figure*}

\begin{theorem}\label{theorem:interaction-potential}
Under dynamics~\eqref{eq:dyn-decoupled}, and cost structures~\eqref{eq:running-cost-revised} and~\eqref{eq:terminal-cost-revised}, if for any two agents $i$ and $j$, we have
\begin{align}\label{eq:collision-conditions}
    C_{ij}^{a}(\state_i,\state_j) = C_{ji}^{a}(\state_j,\state_i),
\end{align}
the underlying differential game $\Gamma_{\state_0}^T$ is a potential differential game with the following potential functions $p$ and~$\Bar{s}$
\begin{equation}\label{eq:interaction-potential}
\begin{aligned}
    p(\state,\action,t) &= \sum_{i=1}^N{C_i^{tr}} (\state_i,\action_i) + 
    \sum_{1\leq i < j} {C_{ij}^a}(\state_i,\state_j),\\
    \Bar{s}&=\sum_i C_{i,\horizon}^{tr}(x_i),
\end{aligned}
\end{equation}
where the second summation in $p$ is over all unordered pairs of distinct agents' indices.
\end{theorem}
\begin{proof}
We prove this by showing that under assumption~\eqref{eq:collision-conditions}, cost structures~\eqref{eq:running-cost-revised} and~\eqref{eq:terminal-cost-revised} satisfy the conditions~\eqref{eq:running-cost-structure} and~\eqref{eq:terminal-cost-strucutre} in Theorem~\ref{theorem:potential-definition}. Let $p$, and $\bar{s}$
be defined as introduced in~\eqref{eq:interaction-potential}. For each agent $i$, we define the agent-specific term $c_i$ in~\eqref{eq:running-cost-structure} to be
\begin{align}
    c_i(\state_{-i},\action_{-i},t)&= 
    -\sum_{j\neq i}C_{j}^{tr}(\state_{j},\action_{j})    
    -\!\!\!\!\sum_{\substack{1\leq j < k \\ j\neq i,\;k\neq i}}{C_{jk}^a}(\state_j,\state_k), \label{eq:agent-specific-running}\\
     s_i(\state_{-i}(\horizon))&=-\sum_{j\neq i}C_{j,\horizon}^{tr}(x_j), \label{eq:agent-specific-terminal}
\end{align}
where the first term in $c_i$ is the negative of the sum of the tracking costs of all agents except for agent $i$, and the second term is the sum of all pairwise inter-agent costs for the unordered pairs of agents $\{j\neq i,k\neq i\}$. Similarly, $s_i$ is the summation over all agents' terminal costs except for the agent $i$'s terminal cost. Note that in the above, $c_i$ and $s_i$ are only a function of the states and inputs of agents other than agent $i$. Using the potential function~\eqref{eq:interaction-potential} and the cost functions~\eqref{eq:agent-specific-running} and~\eqref{eq:agent-specific-terminal}, it is easy to verify that for each agent $i$, we have
\begin{align}
    L_i (\state, \action_i)
    &=C_{i}^{tr}(\state_i,\action_i)+\sum_{j\neq i}^{N}C_{ij}^a(\state_i,\state_j) \\
    &= p(x,u,t) + c_i(x_{-i},u_{-i},t),
\end{align}
and further 
\begin{align}
    S_i(\state(\horizon))=\bar{s}(\state(\horizon))+
    s_i(\state_{-i}(\horizon)).
\end{align}
Hence, using Theorem~\ref{theorem:potential-definition}, our game is a potential differential game. 
\end{proof}






Using Theorem \ref{theorem:interaction-potential}, under assumption~\eqref{eq:collision-conditions}, for finding equilibria of the game $\Gamma_{x_0}^\horizon$, instead of solving a set of coupled optimal control problems~\eqref{def:nash},
we can simply solve the optimal control problem~\eqref{eq:OCP}. Consequently, we propose to solve the following optimal control problem in a receding horizon fashion:
\begin{equation}\label{eq:OCP_discrete}
\begin{split}
    \min_{\action(\cdot)} \quad  
    &\int_{0}^{T}p(\state(t),\action(t),t)dt + \Bar{s}(\state(\horizon))\\
    \quad\textrm{s.t.} \quad 
    \Dot{\state}(t)&=f(\state(t),\action(t),t),\\
    x(0) &= x_0.
\end{split}
\end{equation}
where the potential functions $p$ and $\Bar{s}$ are as defined in~\eqref{eq:interaction-potential}. The significance of this result is that now we can use standard single-agent trajectory optimization algorithms to solve~\eqref{eq:OCP_discrete}. For robots with general nonlinear dynamics, we can utilize any nonlinear trajectory optimization algorithm. We use iterative Linear Quadratic Regulator (iLQR) derived in~\cite{li2004iterative,tassa2012synthesis} for solving~\eqref{eq:OCP_discrete}. We choose iLQR because of its success across different robotic applications~\cite{koenemann2015whole,kitaev2015physics,chen2017constrained}.

\begin{algorithm}
    \caption{Potential iLQR}
    \label{alg:ilqr}
    \begin{algorithmic}[1]
        \State \textbf{Inputs}
        \State system dynamics \eqref{eq:dyn-decoupled}, potential functions~\eqref{eq:interaction-potential}
        \State \textbf{Initialization}
        \State initialize the control input using $u_i(.)= 0, 1 \leq i \leq N$
        \State forward simulate~\eqref{eq:dyn-decoupled} to obtain nominal trajectories $\eta=\{\bar{x}(t),\bar{u}_1(t),\cdots,\bar{u}_N(t)\}_{t\in [0,T]}$
        \While{not converged}
            \State linear approximation of \eqref{eq:dyn-decoupled} around $\eta$
            \State quadratic approximation of~\eqref{eq:interaction-potential} around $\eta$
            \State solve the backward recursion through Ricatti equation and obtain new control policies.
            \State forward simulate the controls and obtain the new nominal trajectories $\eta$.
         \EndWhile
        \State \Return control input $u_i(.)$ for every agent $i$.
    \end{algorithmic}
\end{algorithm}

We start with initializing our controller. Then, starting from an initial condition, we integrate the system dynamics~\eqref{eq:dyn-decoupled} forward in time to obtain a nominal trajectory $\eta=\{\bar{x}(t),\bar{u}_1(t),\cdots,\bar{u}_n(t)\}_{t\in [0,T]}$. We linearize the the system dynamics~\eqref{eq:dynamics} around $\eta$ and further compute a quadratic approximation of the potential function in~\eqref{eq:OCP_discrete} around $\eta$. We use Ricatti equation to solve the resulting approximate Linear Quadratic Regulator problem to obtain a new nominal trajectory and repeat this process until convergence. The outline of our interactive trajectory planning algorithm is summarized in Algorithm~\ref{alg:ilqr}.

In the next two sections, we demonstrate the success of our approach in generating intuitive interactive trajectories in both simulations and experiments.


\section{Simulation Studies}\label{sec:sim}
In this section, we demonstrate the performance of our algorithm in a planar navigation setting involving three agents at an intersection where each agent wants to reach its goal while avoiding collisions with other agents (See Fig.~\ref{fig:3_agents_mpc}).



For each agent $i$, we use the following unicycle dynamics to model our vehicle dynamics:
\begin{equation}
\begin{alignedat}{3}
\Dot{p}_{x,i}&=v_i\cos{\theta_i}, \quad&&\Dot{p}_{y,i}&&=v_i\sin{\theta_i},\\
\Dot{\theta}_i&=\omega_i,        &&\;\;\;\Dot{v}_i&&=a_i,\\
\end{alignedat}
\end{equation}
where $p_{i,x}$ and $p_{i,y}$ are the $x$ and $y$ coordinates of the position of agent $i$ in the 2D plane, $v_i$ is the forward velocity for agent $i$, and $\theta_i$ is the heading of vehicle $i$. For each agent $i$, the state vector is $\state_i=[p_{x,i},p_{x,i},\theta_i,v_i]$, and the input vector is
$\action_i=[\omega_i,a_i]$.
We assume that each agent has a tracking cost $C_i^{\text{tr}}$ in the form of~\eqref{eq:tr_cost} where we let the $Q$ and $R$ matrices be diagonal matrices with different weights on each diagonal entry, and each diagonal entry in Q and R matrices acts as a scaling weight for penalizing the corresponding state or control input. Following~\cite{fridovich2020efficient}, we choose the inter-agent costs $C_{ij}^{a}$ to be the following for any two agents $i$ and $j$:

\begin{equation}\label{collision_unicycle}
C^a_{ij}(\state_i,\state_j)=
    \begin{cases}
    (d_{ij}-d_{\text{prox}})^2 & \text{if }d_{ij}<d_{\text{prox}}\\
    0 & \text{else }
    \end{cases}
\end{equation}
where $d_{ij}$ is the distance between vehicles $i$ and $j$, and $d_{\text{prox}}$ is the threshold distance above which no collision cost is incurred. We keep $d_{\text{prox}}$ to be 2.4 meters in our simulations. 

We run Algorithm~\ref{alg:ilqr} in a receding horizon fashion.
Although we presented our algorithm in continuous time, we solved the backward and forward recursions of Algorithm~\ref{alg:ilqr} in discrete time for our implementation. We set our discrete time intervals to be $0.1$ second. We further let the planning horizon for the receding horizon controller be $1$ second. Fig.~\ref{fig:3_agents_mpc} illustrates the trajectories found by our planning algorithm. 

\begin{figure}
\includegraphics[width=\linewidth]{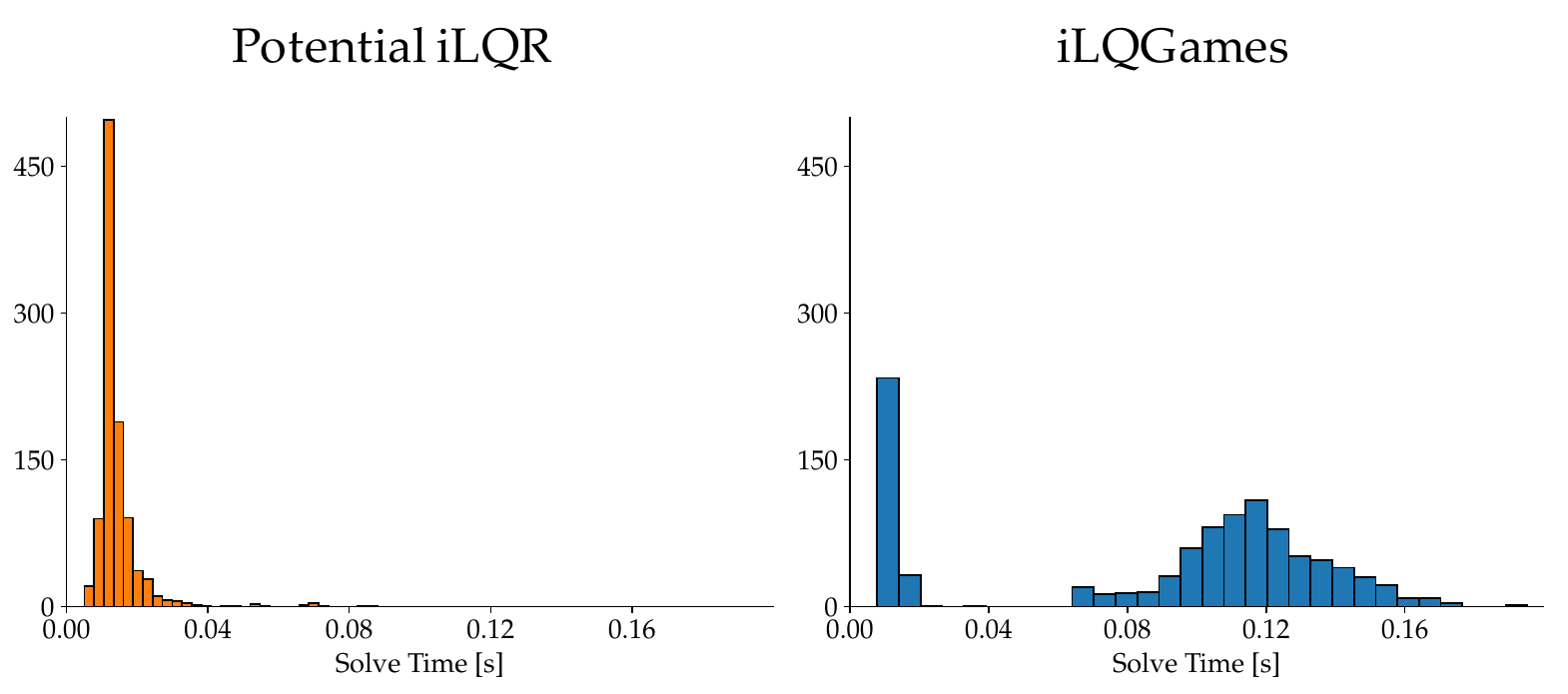}
\caption{Histogram data of our Monte Carlo study with 1000 random initial conditions. Potential-iLQR yields an average solution time of $14$ ms with a standard deviation of $8$ ms. For iLQGames, the average solution time is $89$ ms with a standard deviation of $50$ ms. Our approach is more than 6 times faster than iLQGames.
}
\label{fig:histogram}
\end{figure}

We further compare the performance of our algorithm with the game solver iLQGames~\cite{fridovich2020efficient} by running a Monte Carlo study for the intersection problem. We consider 1000 random initial conditions and evaluate the solution times for both algorithms. We solve for trajectories with a $5$-second prediction horizon. Fig.~\ref{fig:histogram} shows the histograms of the results of our Monte Carlo study. For 1000 samples, the average convergence time of our algorithm is $14$ ms with a standard deviation of $8$ ms, whereas iLQGames has the average convergence time of $89$ ms with a standard deviation of $50$ ms. We also expect our algorithm to be more scalable in the number of agents. Since we solve a single optimal control problem using iLQR, we inherit its $\mathcal{O}(n^3)$ complexity that only scales with the dimension of the state space whereas the iLQGames has $\mathcal{O}(N^3 n^3)$ complexity that scales with both the state-space dimension and the number of agents.

\section{Experiments}\label{sec:exp}
To demonstrate the real-time capabilities of our framework, we set up an experiment in hardware on two Crazyflie 2.0 quadcopters 
within the Robot Operating System (ROS) framework~\cite{quigley2009ros}. 
To obtain the state information, we use a Vicon motion capture system. To send waypoint commands to the quadcopters within the ROS environment, we use the Crazyswarm repository~\cite{preiss2017crazyswarm}. 

\begin{figure*}[h!]
\includegraphics[width=\textwidth]{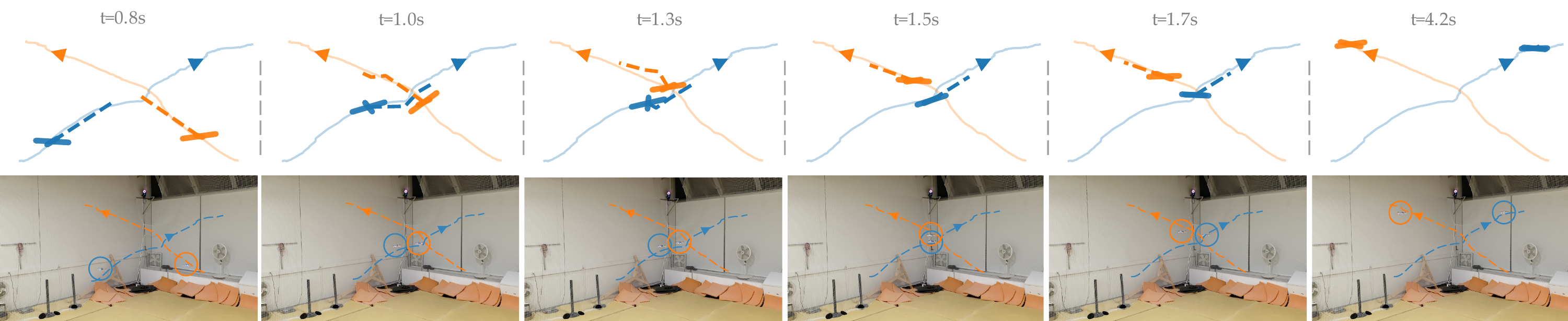}
\caption{The snapshots of the interactive trajectories found by our algorithm for two quadcopters. The start and goal positions of the quadcopters are set such that they need to avoid collisions while traversing their trajectories in 3D. Our algorithm generates intuitive trajectories where quadcopters manage to change altitudes for avoiding collisions. The plots in the first row were generated using the real-time data collected during the experiment. The second row includes the matching video frames of the same experiment. Each column represents the same time instance.}
\label{fig:drones_diagonal}
\end{figure*}

We use a 6D kinematic quadrotor model for quadcopters (see~\cite{sabatino2015quadrotor} for further information). Our state state vector is
\begin{equation*}
    \state_i=[p_{x,i},p_{y,i},p_{z,i},\phi_i,\theta_i,\psi_i],
\end{equation*}
where $p_{x,i},p_{y,i},p_{z,i}$ represent the position of the quadcopter body frame origin and $\phi_i,\theta_i,\psi_i$ represent the roll, pitch, and yaw angles describing the orientation of the body frame of the quadcopter with respect to the inertial frame. For each quadcopter, the control inputs are
\begin{equation*}
    \action_i=[v^b_{x,i},v^b_{y,i},v^b_{z,i},p_i,q_i,r_i],
\end{equation*}
where $v^b_{x,i}$, $v^b_{y,i}$, and $v^b_{z,i}$ are the translational velocities expressed in the body frame. Similarly, $p_i$, $q_i$, and $r_i$ represent the components of the angular velocities expressed in the body frame. 

We run experiments with 2 different scenarios. In the first scenario, the two quadcopters start at the same height and switch their positions such that the starting position of one quadcopter becomes the goal position of the other (see Fig.~\ref{fig:vehicles_front}). We set the xyz coordinate of the start position of one of the quadcopters to be $(0,1,2)$ while the start position of the other quadcopter is  $(1.5,0,2)$. To reach their goal positions, quadcopters need to avoid collisions and coordinate their motion. In the second scenario, the start position of the quadcopters is similar to the first scenario, but their goal positions are in different altitudes. The quadcopters need to traverse three-dimensional trajectories to reach their goals (see Fig.~\ref{fig:drones_diagonal}). In this experiment, one quadcopter starts at $(0,1,1.5)$ and reaches $(1.5,0,2.5)$ while the other quadcopter starts at $(1.5,0,1.5)$ and reaches $(0,1,2.5)$. We let our time step be 0.2 second and set the horizon length to be 1s for our receding horizon planner. 
With this setting, we are able to generate waypoints with 20 Hz update rate. 


Fig.~\ref{fig:vehicles_front} demonstrates the trajectories of the quadcopters in the first scenario, and Fig.~\ref{fig:drones_diagonal} shows the trajectories in the second scenario. As the figures illustrate, the quadcopters successfully avoid collisions and exhibit intuitive trajectories in three-dimensional space. In particular, as Fig.~\ref{fig:vehicles_front} and Fig.~\ref{fig:drones_diagonal} show, the quadcopters change altitudes to avoid collisions with each other. These are indeed an extension of the planar interactive trajectories to the three-dimensional space where the agents can also change altitudes to avoid collisions. Note that both quadcopters maintain their nominal trajectories initially in both scenarios. However, as they get closer than the distance $d_{\text{prox}}$, the coupling between the cost functions becomes effective and the algorithm generates collision-free trajectories.

\section{Conclusion} 
\label{sec:conclusion}

\noindent \textbf{Summary.} We showed that interactive trajectories can be found by solving a single optimal control problem instead of solving a set of coupled optimal control problems. In particular, for a class of multi-agent settings where agents have symmetric mutual cost couplings, we proved that the differential game underlying the interaction is a potential differential game whose equilibrium can be found by solving a single optimal control problem. 
We further showcased the applicability of our framework on a set of simulations and experiments in hardware.

\noindent \textbf{Limitations and Future Work.} 
The idea of reducing equilibria finding to solving a single-agent optimal control problem is achieved under symmetric pairwise couplings. However, it is unclear whether interaction remains a potential game under general asymmetric pairwise couplings. We plan to investigate this further. We believe that for the asymmetric settings, the current algorithm still provides a very fast warm-starting method which provides a feasible relevant initial trajectory. We expect this to address one of the challenges in interactive trajectory planning using game solvers which are sensitive to proper initialization of trajectories.

Our proposed algorithm is a centralized planning algorithm. Ideally, in a game-theoretic setting such as interactions, we want the agents to be the decision makers. Thus, ultimately, we would like to achieve decentralized interactive planning. We expect that our proposed reduction to a single-agent optimal control problem enables us to address this problem by investigating the applications of decentralized control algorithms for minimizing the potential function of the interaction game.

\appendices
\section{Proof of Theorem~\ref{theorem:potential-definition}}\label{sec:appendix}

This theorem was originally proved in~\cite{fonseca2018potential}. For completeness, we are including our proof here. Let $\action^*$ be the solution to \eqref{eq:OCP} and $\state^*$ be the corresponding optimal trajectory of the system state. Fix an agent $i$. Let $\action_i\neq \action^*_i$ be an open-loop control signal
for agent $i$. Let $\state$ be the state trajectory that corresponds to the control signals $(\action_i,\action^*_{-i})$ in the original differential game. Since the state-input trajectory $\state^*$ and $\action^*$ are optimal for \eqref{eq:OCP}, we have
\begin{equation}
\begin{aligned}
    \int_{0}^{\horizon}\;&p\left(\state^*(t),
    \action^*(t),t\right)dt+
    \bar{s}\left(\state^*(\horizon)\right)\\
    &\leq
    \int_{0}^{\horizon}p\left(\state(t),
    \left(\action_i(t),\action^*_{-i}(t)\right),t\right)dt+
    \Bar{s}\left(\state(\horizon)\right).
\end{aligned}
\end{equation}
If we add 
\begin{equation}\label{eq:constant_term}
    \int_{0}^{\horizon}c_i\left(\state_{-i}(t),
    \action^*_{-i}(t),t\right)dt
    +s_i(\state_{-i}(\horizon)),
\end{equation}
to both sides, we have
\begin{equation}\label{eq:nash-definition-proved}
    J_i^*(x_0,\action^*)\leq J_i(x_0,\action_i,\action_{-i}^*).
\end{equation}
It is important to note that because the dynamics~\eqref{eq:dyn-decoupled} are decoupled, $\state_{-i}$ in \eqref{eq:constant_term} depends only on $\action_{-i}^*$. Once $\action_{-i}^*$ is fixed, \eqref{eq:constant_term} becomes a constant term added to both sides. Equation~\eqref{eq:nash-definition-proved} holds for any agent $i$ which is the definition of~\eqref{eq:nash-def} which proves our theorem.

\bibliographystyle{IEEEtran}
\bibliography{references}

\end{document}